\documentclass[sigconf]{acmart}

\AtBeginDocument{%
  }

\copyrightyear{2024}
\acmYear{2024}
\setcopyright{acmlicensed}\acmConference[CIKM '24]{Proceedings of the 33rd
ACM International Conference on Information and Knowledge
Management}{October 21--25, 2024}{Boise, ID, USA}
\acmBooktitle{Proceedings of the 33rd ACM International Conference on
Information and Knowledge Management (CIKM '24), October 21--25, 2024,
Boise, ID, USA}
\acmDOI{10.1145/3627673.3679540}
\acmISBN{979-8-4007-0436-9/24/10}

\usepackage{algorithm}
\usepackage{algorithmic}
\usepackage{subfigure}
\usepackage{multirow}
\usepackage[shortlabels]{enumitem}
\newcommand{\stitle}[1]{\vspace{1ex} \noindent{\bf #1}}

\begin{document}
\title{PSNE: Efficient Spectral Sparsification Algorithms for Scaling Network Embedding} 
	\author{Longlong Lin}
\affiliation{%
	\institution{College of Computer and Information Science, Southwest University}
 \city{Chongqing}
 \country{China}
}
\email{longlonglin@swu.edu.cn}

	\author{Yunfeng Yu}
	\affiliation{%
		\institution{College of Computer and Information Science, Southwest University}
   \city{Chongqing}
 \country{China}
	}
	\email{YunfengYu817@outlook.com}

	\author{Zihao Wang}
\affiliation{%
	\institution{College of Computer and Information Science, Southwest University}
  \city{Chongqing}
 \country{China}
}
\email{zihaowang@outlook.com}

	\author{Zeli Wang}
\affiliation{%
	\institution{Chongqing University of Posts and Telecommunications}
  \city{Chongqing}
 \country{China}
}
\email{zlwang@cqupt.edu.cn}

	\author{Yuying Zhao}
\affiliation{%
	\institution{Vanderbilt University}
  \city{ Nashville}
 \country{USA}
}
\email{yuying.zhao@vanderbilt.edu}
	\author{Jin Zhao}
\affiliation{%
	\institution{Huazhong University of Science and Technology}
  \city{Wuhan}
 \country{China}
}
\email{zjin@hust.edu.cn}

	\author{Tao Jia}\authornote{Corresponding author}
\affiliation{%
	\institution{College of Computer and Information Science, Southwest University}
  \city{Chongqing}
 \country{China}
}
\email{tjia@swu.edu.cn}
\renewcommand{\shortauthors}{Longlong Lin et al.}
\begin{abstract}
Network embedding has numerous practical applications and has received extensive attention in graph learning, which aims at mapping vertices into a low-dimensional and continuous dense vector space by preserving the underlying structural properties of the graph. Many network embedding methods have been proposed, among which factorization of the Personalized PageRank (PPR for short)  matrix has been empirically and theoretically well supported recently.  However, several fundamental issues cannot be addressed. (1) Existing methods invoke a seminal Local Push subroutine to approximate \textit{a single} row or column of the PPR matrix. Thus, they have to execute $n$ ($n$ is the number of nodes) Local Push subroutines to obtain a provable PPR matrix, resulting in prohibitively high computational costs for large $n$. (2) The PPR matrix has limited power in capturing the structural similarity between vertices, leading to performance degradation.  To overcome these dilemmas, we propose PSNE, an efficient spectral s\textbf{P}arsification method for  \textbf{S}caling \textbf{N}etwork \textbf{E}mbedding, which can fast obtain the embedding vectors that retain strong structural similarities. Specifically, PSNE first designs a matrix polynomial sparser to accelerate the calculation of the PPR matrix, which has a theoretical guarantee in terms of the Frobenius norm. Subsequently, PSNE proposes a simple but effective multiple-perspective strategy to enhance further the representation power of the obtained approximate PPR matrix. Finally, PSNE applies a randomized singular value decomposition algorithm on the sparse and multiple-perspective PPR matrix to get the target embedding vectors. Experimental evaluation of real-world and synthetic datasets shows that our solutions are indeed more efficient, effective, and scalable compared with ten competitors.
\end{abstract}


\maketitle
\section{Introduction} 
Graphs are ubiquitous for modeling real-world complex systems, including financial networks, biological networks, social networks, etc. Analyzing and understanding the semantic information behind these graphs is a fundamental problem in graph analysis \cite{DBLP:series/ads/AggarwalW10,DBLP:conf/aaai/LinLJ23,DBLP:conf/icdcs/YuanLK0019,DBLP:journals/eswa/HeLYLJW24,DBLP:journals/corr/abs-2406-07357}. Thus, numerous graph analysis tasks are arising, such as node classification, link prediction, and graph clustering. Due to their exceptional performance, network embedding is recognized as an effective tool for solving these tasks. Specifically, given an input graph $G$ with $n$ nodes, network embedding methods aim at mapping any node $v \in G$ to a  low-dimensional and  continuous dense vector space $x_v \in \mathbb{R}^{k}$ ($k$ is the dimension size and $k<<n$) such that the embedding vectors can unfold the underlying structural properties of graphs.  Thus, the obtained embedding vectors can be effectively applied to these downstream tasks mentioned above \cite{DBLP:series/synthesis/2020Hamilton,DBLP:conf/mir/YuLLWOJ24}.

Many network embedding methods have been proposed in the literature, among which matrix factorization has been empirically and theoretically shown to be superior to  Skip-Gram based random walk methods and deep learning based methods \cite{netmf,DBLP:journals/pvldb/YangSXYB20,sketchNE}, as stated in Section \ref{sec:existing}. Specifically, matrix factorization based solutions first construct a proximity matrix $S$ according to their corresponding applications, in which $S(i, j)$ represents the relative importance of node $j$ with respect to (w.r.t.) node $i$. Then,  the traditional singular value decomposition algorithm is executed on $S$ or some variants of $S$ to obtain the target embedding vectors. Thus,  different proximity matrices were designed and the Personalized PageRank (PPR) matrix emerges as a superior choice due to its good performance \cite{DBLP:journals/pvldb/YangSXYB20,yin2019scalable,DBLP:conf/kdd/ZhangX0H21}. In particular, given two nodes $i, j \in G$, the PPR value $\Pi(i, j)$ is the probability that a random walk starts from $i$ and stops at $j$ using $k$ steps state transition, in which $k$ follows the geometric distribution. Therefore, the PPR values can be regarded as the concise summary of an infinite number of random walks, possessing nice structural properties and strong interpretability for network embedding.

Despite their success, most existing methods suffer from the following several fundamental issues:  (1) \textbf{They either have high
costs to achieve provable performance, or hard to obtain empirical satisfactory embedding vectors for downstream tasks.} Specifically, the state-of-the-art (SOTA) methods, such as STRAP \cite{yin2019scalable} and Lemane \cite{DBLP:conf/kdd/ZhangX0H21}, applied the seminal Local Push subroutine \cite{DBLP:conf/focs/AndersenCL06} to approximate \textit{a single} row or column of the PPR matrix, resulting in that have high overheads, especially for massive graphs. For example, Ref. \cite{DBLP:journals/pacmmod/LiaoLDCQW23a} reported that the current fastest Local Push algorithm takes about 2 seconds on a million-node YouTube graph. Thus,  the Local Push algorithm takes over 23 days to compute the PPR matrix, which is infeasible even for a powerful computing cluster.  For reducing computational costs, NRP \cite{DBLP:journals/pvldb/YangSXYB20} integrated the calculation and factorization of the PPR matrix in an iterative framework.  However, NRP has unsatisfactory theoretical approximation errors and poor empirical performance because it loses the nonlinear capability, as stated in our experiments. (2) Existing methods are highly dependent on the assumption that the PPR matrix can reflect the structural similarity between two nodes \cite{Page1999ThePC}. Namely, larger PPR values, denoted as $\Pi(i, j)$, correspond to the higher structural similarity between nodes $i$ and $j$. \textbf{However, the original PPR matrix has limited power in capturing the
structural similarity between vertices, leading to performance degradation.}  Take Figure \ref {fig:holder2} as an example, we can  see that $\Pi(v_1,v_3)(=0.054)$ is almost three times as many as $\Pi(v_1,v_7)(=0.140)$. Thus, $v_1$ is more similar to $v_7$  than $v_3$ according to the PPR metric. However, $v_1$ and $v_7$ have no common neighbors, and their walking trajectories are not similar. On the contrary,  $v_1$ and $v_3$ share their only neighbor $v_2$ and have almost the same walking trajectories, which is a key property used to characterize whether nodes are in the same cluster\cite{tra1,tra2}. Specifically, let $\mathcal{P}(u)$ be  all walking trajectories starting  with  $u$, we have $\mathcal{P}(v_1)=\{\{v_1,p(v_2)\}|p(v_2) \in \mathcal{P}(v_2)\}$ and  $\mathcal{P}(v_3)=\{\{v_3,p(v_2)\}|p(v_2) \in \mathcal{P}(v_2)\}$. Thus, $\mathcal{P}(v_1)$ and $\mathcal{P}(v_3)$ are identical except for the starting vertex. As a result, by the PPR metric, the classifier tends to mistakenly assign  $v_1$ and $v_7$ to the same class even though $v_1$ and $v_3$ belong to the same community and have stronger structural similarities.  So, exploring alternative methods for overcoming these limitations remains a huge challenge.

\begin{figure}[t]
    \centering
\includegraphics[width=0.42\textwidth]{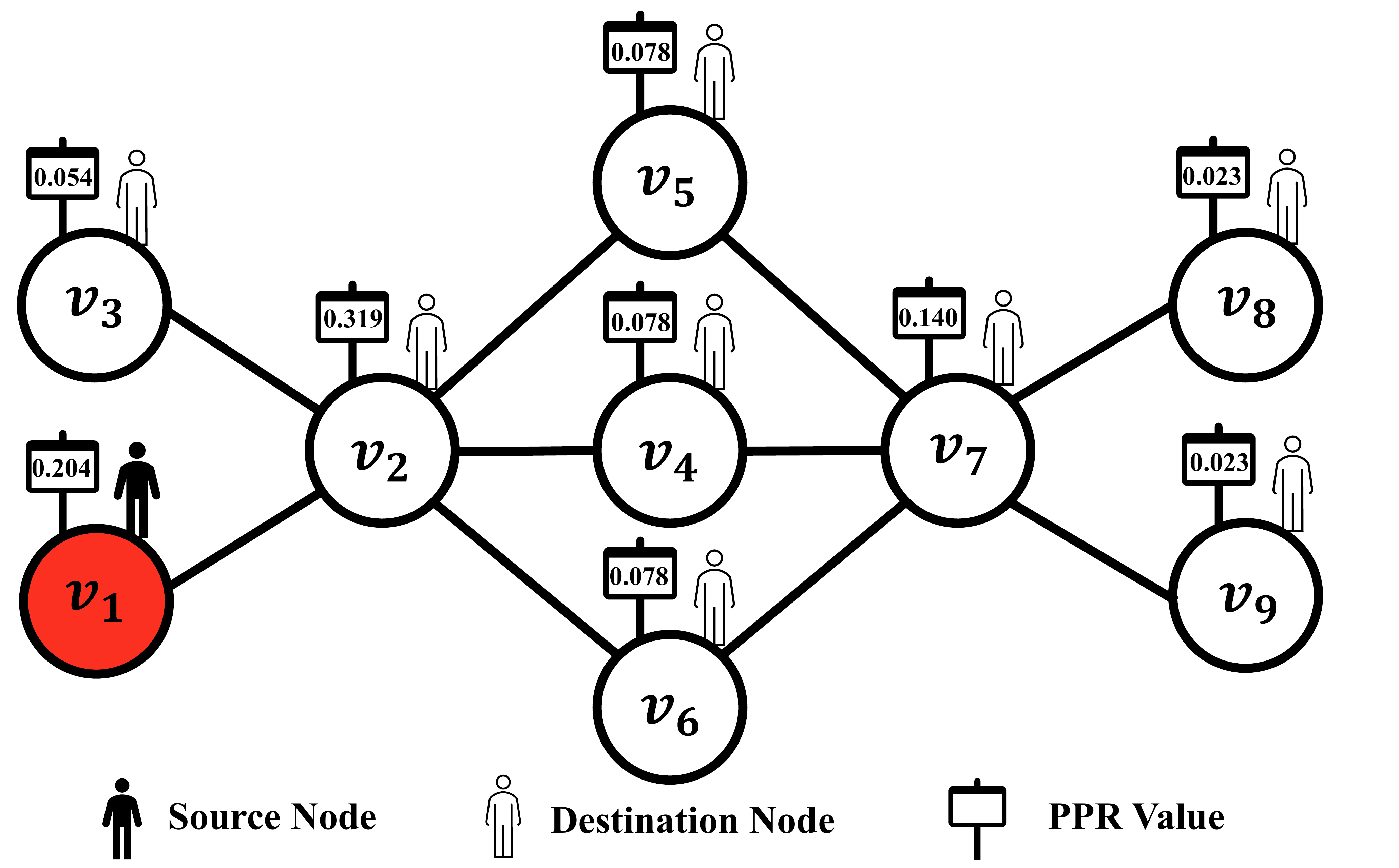}
    \caption{$v_1$ is the source node and $\alpha=0.15$ is the decay parameter in PPR. The PPR values $\Pi(v_1, v_3)$ and $\Pi(v_1, v_7)$ are 0.054 and 0.140, respectively. The proposed multiple-perspective PPR values $M(v_1, v_3)$ and $M(v_1, v_7)$ are 0.142 and 0.095, respectively (see Table \ref{tab:structure-aware} for details).}\vspace{-0.5cm}
    \label{fig:holder2} 
\end{figure} 

To this end, we propose a novel and efficient spectral s\textbf{P}arsification method for \textbf{S}caling \textbf{N}etwork \textbf{E}mbedding (PSNE). Specifically,  PSNE first non-trivially utilizes the theories of spectral sparsification and random-walk matrix polynomials \cite{2011DBLP:journals/siamcomp/SpielmanS11,DBLP:spectraljournals/corr/ChengCLPT15} to \emph{directly} construct a sparse PPR matrix with a theoretical guarantee in terms of the Frobenius norm (Theorem \ref{thm:spectral}), which avoids repeatedly computing each row or column of the PPR matrix, reducing greatly the computational overheads. Then, a simple yet effective multiple-perspective strategy (Section \ref{subsection:mp}) is proposed to further enhance the representation power of the approximate PPR matrix, which can alleviate the inherent defects of the original PPR metric. Finally, PSNE employs a randomized singular value decomposition algorithm to efficiently factorize the sparse and multiple-perspective PPR matrix and obtain high-quality target embedding vectors. In a nutshell, we highlight our contributions as follows.

 \begin{itemize} [leftmargin=8pt, topsep=0pt]
    \item We are the first to adopt the spectral sparsification theory to \emph{directly} approximate the whole PPR matrix, circumventing the expensive costs for a single row or column of the PPR matrix in existing push-based methods. 
     \item We devise a simple but effective multiple-perspective strategy to further enhance the representation power of the approximate PPR matrix.  A striking feature of the strategy is that it also can be generalized to improve the qualities of SOTA baselines.
     \item Empirical results  on real-world and synthetic  datasets show that our proposed PSNE outperforms the quality by at least 2\% than ten competitors in most cases. Besides, PSNE is also more efficient than existing PPR-based methods without sacrificing accuracy, showing a better trade-off between efficiency and accuracy. 
 \end{itemize}

\section{Related Work} \label{sec:existing}
\subsection{Random Walk Based Network Embedding}
Random walk based methods are inspired by the Skip-Gram model \cite{DBLP:journals/corr/abs-1301-3781}. The high-level idea is to obtain embedding vectors by keeping the co-occurrence probability of the vertices on the random walks. The main difference between these methods is how to generate positive samples by different random walk strategies. For example, \textit{DeepWalk} 
 \cite{perozzi2014deepwalk} utilized truncated random walks. \textit{Line}  \cite{tang2015line} and \textit{Node2vec} \cite{grover2016node2vec}  extended \textit{DeepWalk} with more complicated higher-order random walks or  DFS/BFS search schemes. \textit{APP} \cite{zhou2017scalable}  and \textit{VERSE} \cite{tsitsulin2018verse} adopted the $\alpha$-discounted random walks to obtain positive samples. However, a shared challenge of these methods is high computational costs due to the training of the Skip-Gram model.

\subsection{Deep Learning Based Network Embedding} 
Deep learning provides an alternative solution to generate embedding vectors. For example, \textit{SDNE} \cite{wang2016structural} utilized multi-layer auto-encoders to generate embedding vectors. \textit{DNGR} \cite{cao2016deep} combined random walk and deep auto-encoder for network embedding. \textit{PRUNE} \cite{DBLP:conf/nips/LaiHCYL17} applied the Siamese Neural Network to retain both the point-wise mutual information and the PageRank distribution. \textit{GraphGAN} \cite{DBLP:journals/corr/abs-1711-08267} and \textit{DWNS} \cite{DBLP:conf/www/DaiSZLW19} employed generative adversarial networks \cite{8253599} to capture the probability of node connectivity in a precise manner. \textit{AW} \cite{DBLP:conf/nips/Abu-El-HaijaPAA18} proposed an attention model that operates on the power series of the transition matrix. Note that 
although the  Graph Neural Network  (GNN) with feature information  \cite{kipf2016semi} has achieved great success in many tasks, network embedding, which only uses the graph topology like our paper, is still irreplaceable. Specifically, (1) obtaining the rich node feature information is very expensive and even is not always available for downstream tasks, resulting in limited applications \cite{DBLP:journals/pacmmod/Du00H23}. (2) Existing GNNs are typically end-to-end and need different training processes for different
downstream tasks, leading to inflexibility. On the contrary, by focusing on the graph topology, network embedding provides a structure feature for each node, which is independent of downstream tasks \cite{DBLP:conf/iclr/Tang0Y22}. Thus, network embedding provides a trade-off between the accuracy of downstream tasks and the training cost. 
In short, the main bottlenecks of these deep learning methods are high computational costs and labeling costs, which fail to deal with massive graphs.

 \subsection{Matrix Based  Network Embedding} 
Other popular methods are to factorize a pre-defined proximity matrix that reflects the structural properties of the graph. For example, \textit{GraRep}\cite{grarep} performs SVD on the $k$-th order transition matrix. \textit{NetMF} \cite{netmf} demonstrated the equivalence between random walk based methods and matrix factorization based methods. \textit{NetSMF} \cite{netsmf} combined \textit{NetMF} with  sparsification techniques to further improve efficiency of \textit{NetMF}. However,
\textit{NetSMF} cannot effectively capture the non-uniform higher-order topological information because \textit{NetSMF} inherits the defects of \textit{DeepWalk}, resulting in poor practical performance in most cases, as stated in our empirical results.  \textit{ProNE} \cite{zhang2019prone} utilized matrix factorization and spectral propagation to obtain embedding vectors. \textit{STRAP} \cite{yin2019scalable} adopted the PPR matrix as the proximity matrix for improving the performances of \textit{NetMF} and \textit{NetSMF}. However, STRAP applied the seminal Local Push subroutine \cite{DBLP:conf/focs/AndersenCL06} to approximate a single row or column of the PPR matrix, resulting in prohibitively high time\&space overheads.  \textit{HOPE} \cite{ou2016asymmetric}, \textit{AROPE} \cite{zhang2018arbitrary}, \textit{NRP} \cite{DBLP:journals/pvldb/YangSXYB20}, FREDE\cite{DBLP:FREDE}, and SketchNE\cite{sketchNE} derived embedding vectors by implicitly computing the proximity matrix. Thus, they abandoned nonlinear operations on proximity matrices, which limits their representation powers. \textit{Lemane} \cite{DBLP:conf/kdd/ZhangX0H21}  considered the decay factor $\alpha$ in PPR should not be fixed but learnable, resulting in more flexibility. However, this learning process brings  high overheads for \textit{Lemane}.

\section{Preliminaries} \label{sec:pro}
We use $G(V, E)$ to denote an undirected graph, in which $V$ and $E$ are the vertex set and the edge set of $G$, respectively. Let $n=|V|$ (resp., $m=|E|$) be the number of vertices (resp., edges). $\boldsymbol{A}$ is the adjacency matrix with $A_{ij}$ as the element of $i$-th row and $j$-th column of $\boldsymbol{A}$, $\boldsymbol{D}=diag(d_1, ..., d_n)$ is the degree matrix with $d_i = {\sum}_j A_{ij}$, $\boldsymbol{L}=\boldsymbol{D}-\boldsymbol{A}$ be the Laplacian matrix, $\boldsymbol{P}=\boldsymbol{D}^{-1}\boldsymbol{A}$ be the state  transition matrix.
\begin{figure*}[t]
    \centering
    \includegraphics[width=0.95\textwidth]{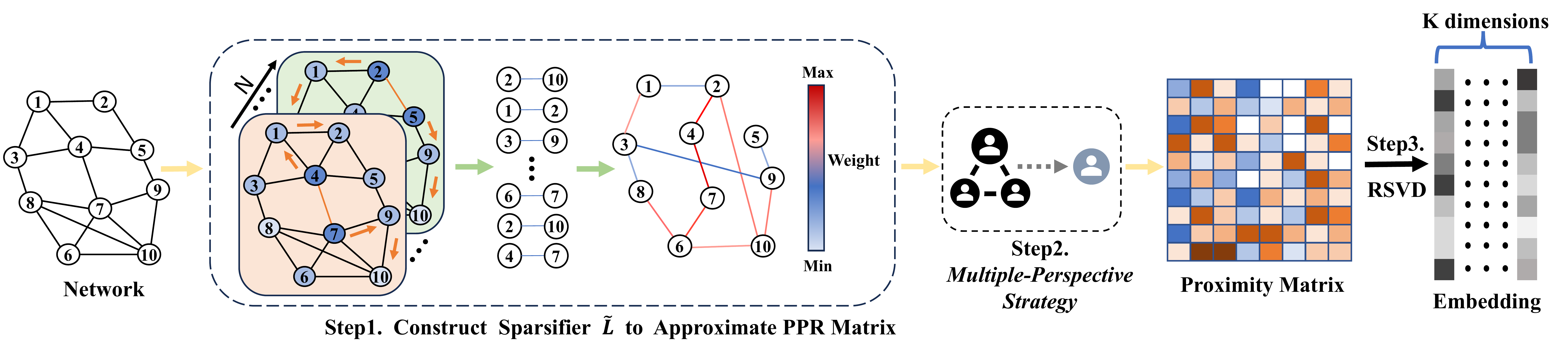}
    \caption{The design of PSNE framework. In Step 1 (i.e., Section \ref{subsec:sppr}), PSNE first constructs the sparsifier $\widetilde{L}$ by sampling $N$ paths and assigning weights to the newly sampled edges. Then,  Equations \ref{fomula4}-\ref{eq:9} and $\widetilde{L}$ are applied to  \emph{directly} approximate the PPR matrix, avoiding repeatedly computing each row or column of the PPR matrix in the traditional Local Push method. In Step 2 (i.e., Section \ref{subsection:mp}), PSNE devises a multiple-perspective strategy to further enhance the representation of the coarse-grained and sparse PPR matrix obtained by Step 1. In Step 3 (i.e., Section \ref{subsec:algorithm}), a randomized singular value decomposition (RSVD) algorithm is executed on the sparse and multiple-perspective PPR proximity matrix to obtain the target embedding matrix. 
}
    \label{fig:holder1}
\end{figure*}

\noindent\textbf{Personalized PageRank (PPR)} is the state-of-the-art proximity metric, which can measure the relative importance of nodes \cite{DBLP:conf/focs/AndersenCL06,DBLP:journals/pvldb/LinYLZQJJ24}. The PPR value ${\Pi}(u,v)$  is the probability that an $\alpha$-decay random walk from $u$ stops at node $v$, in which an $\alpha$-decay random walk has $\alpha$ probability to stop at the current node, or ($1-\alpha$) probability to randomly jump to one of its neighbors. Thus, the length of $\alpha$-decay random walk follows the geometric distribution with success probability $\alpha$. The PPR matrix  $\Pi$ are formulated  as follows:
\begin{equation}\label{fomula1}
\boldsymbol{\Pi}=\sum_{r=0}^{\infty} \alpha(1-\alpha)^{r} \cdot \mathbf{P}^{r}
\end{equation}
\noindent\textbf{Problem Statement.} Given an undirected  graph $G(V,E)$, the network embedding problem aims to obtain a mapping function $f: V \xrightarrow{} \mathbb{R}^{k}$, in which $k$ is a positive integer representing the embedding dimension size and $k<<n$. An effective network embedding function $f$ unfolds the underlying structural properties of graphs.

\section{PSNE: Our Proposed Solution} \label{sec:our}
      Here, we introduce a novel and efficient spectral s\textbf{P}arsification algorithm PSNE for \textbf{S}caling \textbf{N}etwork \textbf{E}mbedding. PSNE first applies non-trivially the spectral graph theories to sparse the PPR matrix with theoretical guarantees. Subsequently, PSNE devises multiple-perspective strategies to further enhance the representation power of the sparse PPR matrix. Finally, a random singular value decomposition algorithm is executed on the refined sparse PPR matrix to obtain target embedding. Figure \ref{fig:holder1} is the framework of PSNE.

\subsection{Spectral Sparsification for PPR Matrix}\label{subsec:sppr}

\begin{definition} [Random-Walk Matrix Polynomials] \label{def:do}
For an undirected graph $G$ and a non-negative vector $\boldsymbol{\beta}=(\beta_1,...,\beta_T)$ with $\sum_{i=1}^T \beta_i=1$, the matrix
\begin{equation}\label{fomula2}
\mathbf{L}_{\boldsymbol{\beta}}(G)=\mathbf{D}-\sum_{r=1}^{T} \beta_{r} \mathbf{D}  \left(\mathbf{D}^{-1} \mathbf{A}\right)^{r}
\end{equation}
is a $T$-degree random-walk matrix polynomial of $G$.
\end{definition}

\begin{theorem} \label{thm:sparsifier} [\textbf{Sparsifiers of Random-Walk Matrix Polynomials}] 
\label{theo:1}
For any undirected graph $G$ and $0 <\epsilon\leq 0.5$, there exists a matrix $\widetilde{L}$ with $O(n \log n/{{\epsilon}^2})$ non-zeros entries such that for any $x\in\mathbb{R}^{n}$, we have
\begin{equation}\label{fomula3}
(1-\epsilon) \boldsymbol{x}^{\top}\widetilde{\mathbf{L}} \boldsymbol{x} \leq \!\boldsymbol{x}^{\top} L_{\boldsymbol{\beta}}(G) \boldsymbol{x}\leq(1+\epsilon)   \boldsymbol{x}^{\top}\widetilde{\mathbf{L}} \boldsymbol{x}
\end{equation}
\end{theorem}

The matrix $\widetilde{L}$ satisfying Equation \ref{fomula3} is called spectrally similar with approximation parameter ${\epsilon}$ to $L_{\boldsymbol{\beta}}(G)$, which can be constructed by the two-stage computing framework \cite{DBLP:spectraljournals/corr/ChengCLPT15}. In the first stage, an initial sparsifier with $O(Tm \log n /\epsilon^{2})$ non-zero entries is found. In the second stage,  a standard spectral sparsification algorithm \cite{2011DBLP:journals/siamcomp/SpielmanS11} is applied in the initial sparsifier to further reduce the number of non-zero entries to $O(n\log n/{{\epsilon}^2})$.  Note that the second stage requires complex graph theory to understand and consumes most of the time of the two-stage computing framework. Thus, in this paper, we first non-trivially utilize the first stage to obtain the coarse-grained sparsifier quickly. Then, we propose an effective multiple-perspective strategy to enhance the representation power of the coarse-grained sparsifier in the next subsection. Specifically, the $T$-truncated PPR matrix is given as follows:
\begin{equation}\label{fomula4}
\boldsymbol{\Pi}^{\prime}=\boldsymbol{\Pi}-\sum_{r=T+1}^{+\infty} \alpha(1-\alpha)^{r}   \mathbf{P}^{r}=\sum_{r=0}^{T} \alpha(1-\alpha)^{r}   \mathbf{P}^{r}
\end{equation}
where $T$ is the truncation order ($T$ is also $T$ in Definition \ref{def:do}). By combining Equation \ref{fomula2} and Equation \ref{fomula4}, we have 
\begin{align}
   \boldsymbol{\Pi}^{\prime}&={\alpha}\boldsymbol{I}+\sum_{r=1}^{T} \alpha(1-\alpha)^{r}   \mathbf{P}^{r}\\
   &={\alpha}\boldsymbol{I}+\boldsymbol{D^{-1}}\cdot \sum_{r=1}^{T} \alpha(1-\alpha)^{r}   \boldsymbol{D}\mathbf{P}^{r}\\ 
      &={\alpha}\boldsymbol{I}+{\alpha}_{sum}\boldsymbol{D^{-1}}\cdot \sum_{r=1}^{T} \alpha(1-\alpha)^{r}/{\alpha}_{sum}  \boldsymbol{D}\mathbf{P}^{r}\\ 
   &={\alpha}\boldsymbol{I}+{\alpha}_{sum}\boldsymbol{D^{-1}}\cdot (\boldsymbol{D}-\mathbf{L}_{\boldsymbol{\beta}}(G))\\ 
      &={\alpha}\boldsymbol{I}+{\alpha}_{sum}\cdot (\boldsymbol{I}-\boldsymbol{D^{-1}}\mathbf{L}_{\boldsymbol{\beta}}(G)) \label{eq:9}
\end{align}

Where ${\alpha}_{sum}=\sum_{i=1}^{T}\alpha(1-\alpha)^{i}$. Therefore, we establish a theoretical connection between the (truncated) PPR matrix and random-walk matrix polynomials. Based on this connection, we devise a novel sparsifier to obtain the approximate PPR matrix (Algorithm \ref{alg:sparsifier}), reducing greatly the prohibitively high computational cost of existing local push-based embedding methods. Specifically, Algorithm \ref{alg:sparsifier} first initializes an undirected graph  $\tilde{\mathbf{G}}=(V, \emptyset)$, in which $V$ is the vertex set of the input graph $G$ (Line 1). Subsequently, Lines 2-7 of Algorithm \ref{alg:sparsifier} adds $O(N)$ edges to $\tilde{\mathbf{G}}$ by executing iteratively the Path\_Sampling Function (Lines 11-18). Finally,  Algorithm \ref{alg:sparsifier}  applies Equations \ref{fomula4}-\ref{eq:9} to get an approximate PPR matrix $\boldsymbol{\tilde{\Pi}}$ with $O(N)$ non-zeros entries (Lines 8-10).

\stitle{\textbf{Remark.}} The path length $r$ in Line 4 is selected with the probability $\alpha(1-\alpha)^{r}/\sum_{i=1}^{T}\alpha(1-\alpha)^{i}$ for satisfying the condition of Definition \ref{def:do}, that is $\alpha_r$ = $\alpha(1-\alpha)^{r}/\sum_{i=1}^{T}\alpha(1-\alpha)^{i}$ and $\sum_{r=1}^{T}\alpha_r=1$. As a result, Algorithm \ref{alg:sparsifier}  can obtain a sparse PPR matrix with a theoretical guarantee in terms of the Frobenius norm, which will be analyzed theoretically later. Besides, Algorithm \ref{alg:sparsifier} also leverages the intuition that closer nodes exhibit a higher propensity for information exchange. Therefore, the shorter the random walk, the greater the probability of being selected to promote local interactions, enabling more effective capture non-uniform high-order structural proximities among vertices for obtaining high-quality embedding vectors, which is verified in our experiments.

\begin{algorithm}[t] 
\caption{\textit{Sparsifier of the PPR Matrix}}
	\label{alg:sparsifier}
 	\begin{algorithmic}[noline]
	\STATE \textbf{Input}: 	An undirected graph $G(V, E)$; the truncation order $T$; the number of non-zeros $N$ in the sparsifier; the decay factor $\alpha$ of PPR \\
 \textbf{Output}: A sparse PPR matrix $\boldsymbol{\tilde{\Pi}}$
	\end{algorithmic}
\begin{algorithmic}[1] 
            \STATE Initializing an undirected   graph $\tilde{\mathbf{G}}=(V, \emptyset)$
		\FOR{$i=1$ to $N$}
            \STATE Uniformly pick an edge $e=(u,v) \in E$ 
            \STATE Pick an integer $r\in[1,T]$ with a probability of $\alpha(1-\alpha)^{r}/\sum_{i=1}^{T}\alpha(1-\alpha)^{i}$ 
            \STATE $u',v',Z_p$$\leftarrow$ Path\_Sampling$(e,r)$
            \STATE Add weight $2rm/(NZ_p)$ to the edge ($u',v'$) of $\tilde{\mathbf{G}}$
            \ENDFOR
            \STATE $\tilde{\mathbf{L}} \leftarrow$  the  Laplacian matrix of $\tilde{\mathbf{G}}$
            \STATE  $\boldsymbol{\tilde{\Pi}}\leftarrow {\alpha}\boldsymbol{I}+{\alpha}_{sum}(\boldsymbol{I}-\boldsymbol{D^{-1}}\tilde{L})$
            \RETURN $\boldsymbol{\tilde{\Pi}}$\\
\STATE \stitle{Function} Path\_Sampling$(e=(u,v),r):$
            \STATE \quad  Uniformly pick an integer $j\in[1,r]$\\
		  \STATE  \quad Perform ($j-1$)-step random walk from  $u$ to $n_0$ 
            \STATE  \quad Record anonymous trajectory from  $u$ to $n_0$ ($AnoTra_u$)
            \STATE  \quad Perform ($r-j$)-step random walk from  $v$ to $n_r$
            \STATE  \quad Record anonymous trajectory from  $v$ to $n_r$ ($AnoTra_v$)
            \STATE  \quad Calculate pattern similarity via $AnoTra_u$ and $AnoTra_v$ for the multiple-perspective strategy of Section \ref{subsection:mp}
        \STATE     \quad \textbf{return} $n_0$, $n_r$, $\sum_{i=1}^{r} \frac{2}{\textbf{A}{(n_{i-1},n_i)}}$
            \end{algorithmic}
\end{algorithm}

\begin{table*}[t!]
\caption{Illustration of the original PPR and Multiple-Perspective PPR
(MP-PPR) values on Figure  \ref{fig:holder2}.}
\centering
\scalebox{1}{
\begin{tabular}{|c|c|c|c|c|c|c|c|c|c|}
\hline
 &$v_1$ & $v_2$ & $v_3$ & $v_4$ & $v_5$ & $v_6$ & $v_7$ & $v_8$ & $v_9$ \\ \hline
PPR    &0.204    & 0.319  & 0.054  & 0.078  & 0.078  & 0.078  & 0.140  & 0.023  & 0.023  \\ \hline

MP-PPR &0.142 & 0.180  & 0.142  & 0.145  & 0.145  & 0.145  & 0.095  & 0.062  & 0.062  \\ \hline
\end{tabular}
}
\label{tab:structure-aware}
\end{table*}

\subsection{From PPR to Multiple-Perspective PPR}  \label{subsection:mp}

As depicted in Figure \ref{fig:holder2}, the original PPR metric cannot effectively capture the structural similarity between vertices. Besides, if we directly take the approximate PPR matrix obtained by Section \ref{subsec:sppr} as the input to matrix factorization, there will also be performance degradation. To overcome these issues, we propose a simple yet effective multiple-perspective strategy to achieve the following goals: (1) Alleviating the inherent defects of the original PPR metric; (2) Enhancing the representation power of the coarse-grained and sparse PPR matrix obtained in Section \ref{subsec:sppr}.

 Through deep observation as shown in Figure \ref{fig:holder2}, 
we found that the original PPR measures node pair proximity from a \textit{single perspective} and ignores \textit{pattern similarity}. For example, in a social network with users $A$, $B$, and $C$. $C$ is the close friend of $A$, and $A$'s assessment of $B$ is not only influenced by $A$'s subjective impression of $B$ but also indirectly influenced by $C$'s impression of $B$. However, the original PPR metric only considers $A$'s impression of B.
On top of that, the original PPR fails to illustrate that node $v_1$ exhibits a structural pattern more similar to nodes $v_3$ rather than $v_7$.

Based on these intuitions, we propose a novel metric of Multiple-Perspective PPR (MP-PPR)  to effectively compute the structural proximity between destination nodes and the source node with pattern similarity. Since considering the perspective of all nodes is computationally expensive with little performance improvement, we integrate only the perspectives of the one-hop neighbors of the source node. Consequently, the multiple-perspective proximity of destination node $j$ w.r.t. the source node $i$ is stated as follows.
\begin{equation}\label{fomula6}
M(i,j)_{S} = \sum_{h\in \mathcal{N}(i)} \lambda_{hi}S_{hj}+ \lambda_{ii}S_{ij}
\end{equation}
where $S$ is any proximity matrix (e.g., the PPR matrix) and $\mathcal{N}(i)$ represents the neighbor set of  $i$.  $\lambda_{hi}$ (resp., $\lambda_{ii}$) is the multiple-perspective coefficient of node $h$ (resp., $i$) to node $i$. In this paper, we take $\lambda_{hi}=wp_{hi}/(\sqrt{(d_{h}+1)}\sqrt{(d_{i}+1)})$ and $\lambda_{ii}=1/(d_{i}+1)$ where $wp_{hi}$ is pattern similarity between $(v_h,v_j)$.
To characterize the pattern similarity between nodes, we introduce the well-known anonymous walk \cite{ivanov2018anonymous} as follows.
\begin{definition} [Anonymous Walk] \label{Anonymous}
If $A = (v_1,v_2,...,v_n)$ is a random walk trajectory, then its corresponding anonymous walk is the sequence of integers ${AnoTra}_{A} = (f(v_1),f(v_2),..,f(v_n))$, where $f(.)$ is a mapping that maps nodes to positive integers.
\end{definition}
Different nodes on the same trajectories are mapped to different positive integers, which may coincide on different paths. For example, trajectories $P_A = (v_1, v_2, v_3, v_2, v_3)$ and $P_B = (v_3,v_4,v_2,v_4, v_2)$ share the common anonymous trajectory $(1, 2, 3, 2, 3)$.
We utilize the anonymous walk for pattern similarity calculation, which demands extensive trajectory sampling, limiting scalability.
However, we have identified specific features in Equation \ref{fomula6} and Algorithm \ref{alg:sparsifier} as follows. (1) Equation \ref{fomula6} focuses on pattern similarity of neighboring nodes within one hop, reducing initial sampling needs. (2) Algorithm \ref{alg:sparsifier} already includes extensive path sampling, allowing for reduced sampling in anonymous random walks through strategic design (Lines 14 and 16 of Algorithm\ref{alg:sparsifier}).
 
After obtaining the anonymous trajectories, we utilize the Longest Common Subsequence \cite{lccs} to ascertain the similarity between the two trajectories.
\begin{theorem} \label{thm:sparsifier} [ \textbf{Longest Common Subsequence(LCSS)}] 
\label{theo:LCSS}
For two trajectories $P_A= (a_1,a_2,...,a_n)$ and $P_B= (b_1,b_2,...,b_m)$ with lengths n and m respectively, where the length of the longest common subsequence is:
\begin{equation}\label{lcss} \scriptsize
\operatorname{LCSS}(P_A, P_B)=\left\{\begin{array}{ll}
0  \text { if } P_A=\varnothing \text { or } P_B=\varnothing \\
1+\operatorname{LCSS}\left(a_{t-1}, b_{i-1}\right),   \text { if } a_{t}=b_{i} \\
\max \left(\operatorname{LCSS}\left(a_{t-1}, b_{i}\right), \operatorname{LCSS}\left(a_{t}, b_{i-1}\right)\right),   \text { otherwise }
\end{array}\right.
\end{equation}
where $t=1,2,...,n$ and $i=1,2,...,m$ and $\varnothing$ is empty trajectory.
\end{theorem}
Therefore, the pattern similarity $w_{Pattern}$ in Equation \ref{fomula6} is defined via anonymous walk paths and LCSS as follows:
\begin{equation}\label{wpattern}
wp_{ij} = \frac{1}{s}\sum^s_1 LCSS({AnoTra}_i,{AnoTra}_j)/ {lenth(AnoTra_i)}
\end{equation}
where $AnoTra_i$ and $AnoTra_j$ are anonymous random walk trajectories starting from nodes $i$ and $j$, respectively. $lenth(.)$ is a function of trajectory length and $s$ is the sampling numbers node pair $(i,j)$.

As shown in Table \ref{tab:structure-aware}, $v_3$ exhibits a higher MP-PPR value than $v_7$.  
Thus, MP-PPR captures more reasonable proximity for network embedding from multiple perspectives without compromising the proximity between different nodes as reflected in the original PPR.

\subsection{Our PSNE and Theoretical Analysis} \label{subsec:algorithm}
\begin{algorithm}[t] 
	\caption{\textit{PSNE}}
	\label{alg:final}
 	\begin{algorithmic}[noline]
	\STATE \textbf{Input}: An undirected graph $G(V,E)$; the truncation order $T$; the number of non-zeros $N$ used in the PPR matrix sparsifier;  the decay factor $\alpha$ of PPR; the filter parameter $\mu$; the embedding dimension size $k$
	\STATE \textbf{Output}: The network embedding matrix
	\end{algorithmic}
	\begin{algorithmic}[1] 
            \STATE  Obtaining a sparse PPR matrix $\boldsymbol{\tilde{\Pi}}$ by executing Algorithm \ref{alg:sparsifier}
            \STATE  Obtaining the multiple-perspective PPR $\boldsymbol{M}_{\boldsymbol{\tilde{\Pi}}}$ by Equation \ref{fomula6} 
                        \STATE $M_{\tilde{\Pi}}\leftarrow \sigma_{\mu}(M_{\tilde{\Pi}})$  \quad // $\sigma_{\mu}$ is a non-linear activation function
            \STATE  $U,\Sigma,V \leftarrow$ Randomized SVD $(\tilde{M}_{\tilde{\Pi}}, k)$
            \RETURN $U\sqrt{\Sigma}$  as the network embedding matrix
	\end{algorithmic}
\end{algorithm}

Based on the above theoretical backgrounds, we devise an efficient spectral sparsification algorithm for scaling network embedding (Algorithm \ref{alg:final}). Firstly, Algorithm \ref{alg:final} obtains a sparse PPR matrix with a theoretical guarantee in terms of the Frobenius norm (Line 1). Subsequently, it obtains the multiple-perspective PPR matrix $\boldsymbol{M}_{\boldsymbol{\tilde{\Pi}}}$ (Line 2). 
Finally, Lines 3-5 obtain the network embedding matrix by executing the randomized singular value decomposition (RSVD) algorithm \cite{halko2010finding}.  Here, $\sigma_{\mu}$ is a non-linear activation function (e.g., $\sigma_{\mu} (x)=max(0,\log(xn\mu))$ \cite{yin2019scalable,DBLP:conf/kdd/ZhangX0H21,sketchNE}) with the filter parameter $\mu$. Next, we analyze the time\&space complexities of the proposed \textit{PSNE} and the corresponding approximation errors. 
\begin{theorem} \label{thm:complexity}
     The time complexity and space complexity of Algorithm \ref{alg:final} are $O(m \log n +mk+nk^2)$ and $O(m\log n+nk)$, respectively.
\end{theorem}
\begin{proof}
PSNE (i.e., Algorithm \ref{alg:final}) has three main steps as follows:
\begin{itemize} [leftmargin=8pt, topsep=0pt]
    \item Step 1 (i.e., Algorithm 1): Random-Walk Molynomial Sparsifier for PPR matrix. Lines 1-7 of Algorithm 1 sample $O(N)$ paths to construct $O(N)$ edges for the sparse graph $\tilde{G}$. The expected value of $r$ ($r$ is the length of sample path), denoted as $\mathbb{E}(r)$, is given by $\sum_{i=1}^{T}i(\alpha(1-\alpha)^{i}/\sum_{i=1}^{T}\alpha(1-\alpha)^{i})  \leq 1/\alpha$. As a result, Lines 1-7 of Algorithm 1 consume $O(N{\mathbb{E}}(r))=O(N/\alpha)$ time. In Lines 8-9, Algorithm 1 consumes $O(N)$ time to compute the spare PPR matrix $\tilde{\Pi}$. Thus, the time complexity of Algorithm 1 is $O(N/\alpha)$.  For space complexity, Algorithm 1 takes $O(N)$ extra space to store graph $\tilde{\mathbf{G}}$ and matrix $\tilde{\Pi}$. So, the space complexity of Algorithm 1 is $O(N+n+m)$.

    \item Step 2 (i.e., Lines 2-3 of Algorithm 2): Multiple-Perspective strategy. In particular, according to Equation 10, we can know that this step consumes  $O(mx_{avg})$ time to obtain MP-PPR matrix $\boldsymbol{M}_{\boldsymbol{\tilde{\Pi}}}$, in which $x_{avg}=\frac{N}{n}$ is the average number of non-zero elements per row of $\boldsymbol{\tilde{\Pi}}$. Thus, the time complexity of step 2 is $O(m+m\frac{N}{n})$.  In most real-life graphs, $m=O(n \log n)$, thus, the time complexity of step 2 can be further reduced to $O(m+N\log n)$. The space complexity of step 2 is $O(m+N\log n +n)$.

    \item Step 3 (i.e., Lines 4-5 of Algorithm 2): Randomized Singular Value Decomposition.   By \cite{halko2010finding}, we know that this step needs $O(Nk + nk^2 + k^3)$ time and $O(N + nk)$ space to get the network embedding matrix.
    \end{itemize}

In a nutshell, the time (resp., space)  complexity of \textit{PSNE} is $O(N/\alpha+m+N \log n +Nk+nk^2+k^3)$ (resp., $O(m+N \log n+nk)$). Following the previous methods \cite{2011DBLP:journals/siamcomp/SpielmanS11,DBLP:journals/pvldb/YangSXYB20}, $\alpha$ is a constant and $N=O(m)$, thus, the time (resp., space)  complexity of \textit{PSNE} can be  reduced to $O(m \log n +mk+nk^2)$ (resp., $O(m\log n+nk)$). Thus, we have completed the proof of Theorem \ref{thm:complexity}.
\end{proof}
Missing proofs are deferred to our Appendix Section.
\begin{theorem} \label{thm:spectral}
Let $\Pi$ be the exact PPR matrix (i.e., Equation \ref{fomula1}) and $\tilde{\Pi}$ be the approximate 
sparse matrix obtaind by Algorithm \ref{alg:sparsifier}, we have ${\Vert\Pi-\tilde{\Pi}\Vert}_F \leq 
\sqrt{n}((1-\alpha)^{T+1}+4\epsilon\cdot {\alpha}_{sum})$.
\end{theorem}

\begin{theorem} \label{thm:4}
Let $M_{\boldsymbol{\Pi}}$ (resp., $M_{\boldsymbol{\tilde{\Pi}}}$) be the MP-PPR matrix by executing Equation \ref{fomula6} on $\Pi$ (resp., $\tilde{\Pi}$), we have
${\Vert\boldsymbol{\sigma}(M_{\boldsymbol{\Pi}}, \mu)-\boldsymbol{\sigma}(M_{\boldsymbol{\tilde{\Pi}}}, \mu)\Vert}_F \\\leq 
n((1-\alpha)^{T+1}+4\epsilon\cdot {\alpha}_{sum})$.
\end{theorem}

\section{Empirical Results}\label{sec:experiments} 

In this section, we answer the following Research Questions:

\begin{itemize} [leftmargin=8pt, topsep=0pt]
\item \textbf{RQ1:} How much improvement in effectiveness and  efficiency is our  \emph{PSNE} compared to other baselines? 
\item \textbf{RQ2:} Whether the \emph{multiple-perspective} strategies  can be integrated into other baselines to improve qualities?
\end{itemize}

\subsection{Experimental Setup}

\textbf{Datasets.}We evaluate our proposed solutions on several publicly-available datasets (Table \ref{table:datasets}), which are widely used benchmarks for network embedding \cite{DBLP:conf/kdd/ZhangX0H21,sketchNE,DBLP:FREDE,DBLP:journals/pvldb/YangSXYB20}. 
BlogCatalog, Flickr, and YouTube are undirected social networks where nodes represent users and edges represent relationships. The Protein-Protein Interaction (PPI) dataset is a subgraph of the Homo sapiens PPI network, with vertex labels from hallmark gene sets indicating biological states. The Wikipedia dataset is a word co-occurrence network from the first million bytes of a Wikipedia dump, with nodes labeled by Part-of-Speech tags.

\stitle{Baselines and  parameters.}  The
following ten competitors are implemented for comparison: DeepWalk \cite{perozzi2014deepwalk}, Grarep \cite{grarep}, HOPE \cite{hope}, NetSMF \cite{netsmf}, ProNE \cite{zhang2019prone}, STRAP \cite{yin2019scalable}, NRP \cite{DBLP:journals/pvldb/YangSXYB20}, Lemane \cite{DBLP:conf/kdd/ZhangX0H21}, FREDE \cite{DBLP:FREDE}, and SketchNE \cite{sketchNE}. Note that STRAP, NRP, and Lemane are PPR-based embedding methods. For the ten competitors, we take their corresponding default parameters. The detailed parameter settings of the proposed PSNE are summarized in Table \ref{tab:table-params1}. All experiments are conducted on a Ubuntu server with  Intel (R) Xeon (R) Silver 4210 CPU (2.20GHz) and 1T RAM.

\begin{table}[t]  
\caption{Dataset statistics.}
\centering
\scalebox{1}{
\begin{tabular}{l|l|l|l}
\hline
Dataset& $|V|$ & $|E|$ & \#labels \\ \hline
\multicolumn{1}{l|}{PPI\cite{zhang2019prone}}         & \multicolumn{1}{l|}{3,890}  & \multicolumn{1}{l|}{76,584}  & 50                \\ \hline
\multicolumn{1}{l|}{Wikipedia\cite{zhang2019prone}} & \multicolumn{1}{l|}{4,777} & \multicolumn{1}{l|}{184,812} & 39                \\ \hline
\multicolumn{1}{l|}{BlogCatalog\cite{zhang2019prone}} & \multicolumn{1}{l|}{10,312} & \multicolumn{1}{l|}{333,983} & 39                \\ \hline
\multicolumn{1}{l|}{Flickr\cite{yin2019scalable}}      & \multicolumn{1}{l|}{80,513} & \multicolumn{1}{l|}{5,899,882}   & 195               \\ \hline
\multicolumn{1}{l|}{Youtube\cite{zhang2019prone}}     & \multicolumn{1}{l|}{1,138,499}  & \multicolumn{1}{l|}{2,990,443}   & 47                \\ \hline  
\end{tabular}}
\label{table:datasets}
\end{table}

\subsection{Effectiveness Testing}
We apply node classification to evaluate the effectiveness of our solutions. Node classification aims to accurately predict the labels of nodes. Specifically, a node embedding matrix is first constructed from the input graph. Subsequently, a one-vs-all logistic regression classifier is trained using the embedding matrix and the labels of randomly selected vertices. Finally,  the classifier is tested with the labels of the remaining vertices. The training ratio is adjusted from 10\% to 90\%. To be more reliable,  we execute each method five times and report their  Micro-F1 and Macro-F1 in Figure \ref{fig:nodeclass}\footnote{For Youtube, since FREDE, GraRep, and HOPE cannot obtain the result within 48 hours or out of memory, we ignore their results.}. As can be seen, we can obtain the following observations: 
(1)  Under different training ratios, our PSNE consistently achieves the highest Micro-F1 scores on four of the five datasets, and the highest Macro-F1 scores on PPI, BlogCatalog, and YouTube. For example, on YouTube, PSNE is 1.5\% and 0.9\% better than the runner-up in Micro-F1 and Macro-F1 scores, respectively.   
(2) The Micro-F1  and Macro-F1 
of all baselines vary significantly depending on the dataset and training ratio. For example, DeepWalk outperforms other methods on Flickr (PSNE is the runner-up and slightly worse than DeepWalk) but has inferior Micro-F1 scores on other datasets. HOPE achieves the highest Macro-F1 score on Wikipedia but performs poorly on the Micro-F1 metric. (3) PSNE outperforms other PPR-based methods, including STRAP, NRP, and Lemane, with a margin of at least 2\% in most cases. Specifically, for the Micro-F1 metric, our \emph{PSNE} achieves improvements of 6\%, 3\%, 8\%, 7\%, and 2\% over \emph{NPR} on PPI, Wikipedia, BlogCatalog, Flickr, and Youtube, respectively. For example, for Flickr with more than a few million edges, our \emph{PSNE} achieves 41\% while \emph{NPR} is 34\% in the micro-F1 metric. 
For the Macro-F1 metric, PSNE surpasses all other PPR-based algorithms, achieving a notable lead of 0.5\%, 0.6\%, 2\%, 4\%, and 2\% over the runner-up PPR-based algorithm on these five datasets, respectively. These results show that PSNE's multi-perspective strategy indeed can enhance the representation power of the original PPR matrix.
Besides, these results give clear evidence that our PSNE has high embedding quality compared with the baselines.

\begin{table}[t]
  \caption{Parameter settings of our proposed PSNE.}
  \label{tab:table-params1}
  \centering
  \begin{tabular}{l|l}
    \hline
    Datasets      & Parameters \\
    \hline
 PPI& $\alpha$=0.35, $T$=10, $c$=25, $\mu$=10\\
    \hline
Wikipedia   & $\alpha$=0.50, $T$=10, $c$=25, $\mu$=0.2\\
    \hline
 BlogCatalog  & $\alpha$=0.35, $T$=10, $c$=35, $\mu$=25\\
    \hline
Flickr &  $\alpha$=0.35, $T$=5, $c$=25, $\mu$=1\\
    \hline
   Youtube   &  $\alpha$=0.35, $T$=5,  $c$=30, $\mu$=1\\
    \hline
    
\end{tabular}
\end{table}

\subsection{Efficiency Testing}
For efficiency testing, we do not include non-PPR-based algorithms because all of them are outperformed by \emph{Lemane} and \emph{NRP}, as reported in their respective studies. On top of that, our proposed PSNE is a PPR-based method, so we test the runtime of other PPR-based methods (i.e., STRAP, NRP, and Lemane) for comparison. Table \ref{table:time} presents the wall-clock time of each PPR-based method with 20 threads. As can be seen, NRP outperforms other methods (but it has poor node classification quality and improved by 2\% $\sim$ 25\% by  PSNE, as stated in Figure \ref{fig:nodeclass}), and PSNE is runner-up and slightly worse than NRP. The reasons can be explained as follows: NRP integrates the calculation and factorization of the PPR matrix in an iterative framework to improve efficiency but lacks the nonlinear representation powers for node embedding. However, our PSNE devises the sparsifiers of random-walk matrix polynomials for the truncated PPR matrix, avoiding
repeatedly computing each row or column of the PPR matrix in the push-based methods (e.g., STRAP, Lemane). These results show that PSNE achieves significant speedup with high embedding quality compared with the baselines, which is consistent with our theoretical analysis (Section \ref{subsec:algorithm}).

\subsection{Scalability Testing on Synthetic Graphs}
We use the well-known NetworkX Python package \cite{nx} to generate two types of synthetic graphs \textit{ER} \cite{ER} and \textit{BA} \cite{BA} to test the scalability of our PSNE.
Figure \ref{fig:time_syn} only presents the results of Deepwalk, STRAP, HOPE, and our PSNE, with comparable trends across other methods. By Figure
\ref{fig:time_syn}, we can know that when the number of nodes is small, the DeepWalk and HOPE have a runtime comparable to PSNE and STRAP. However, as the number of nodes increases, the runtime of DeepWalk and HOPE significantly rises, surpassing that of PSNE by an order of magnitude. Additionally, the increase in STRAP's runtime is greater than that of PSNE. These results indicate that our PSNE has excellent scalability
over massive graphs while the baselines do not.

\begin{table}[t!]
 \caption{Runtime (seconds) of different PPR-based methods.} 
	\centering
	\scalebox{1}{
		\begin{tabular}{c|ccccccc}
			\toprule
			\multicolumn{1}{c|}	{Dataset}  &STRAP& NRP&Lemane&PSNE\\
			\midrule
			\multirow{1}{*}{PPI}  & 4.8e1&2e0&1.9e2&2.4e1\\
				
			\midrule
			\multirow{1}{*}{Wikipedia}  &1.1e2&3e0&5.3e2&4.4e1\\
				
			\midrule
			\multirow{1}{*}{BlogCatalog}  &3.7e2&1.1e1&7.4e3&2.6e2\\
			
			\midrule
			\multirow{1}{*}{Flickr}   &5.0e3&2.0e2&1.8e4&2.3e3\\
			
			\midrule
			\multirow{1}{*}{Youtube}  &2.1e4&1.8e2&$>$24h&7.0e3\\
			\bottomrule	
	\end{tabular}}
  	\label{table:time}
\end{table}

\begin{figure}[t!]
	\centering
	\subfigure[\textit{ER}]{
		\includegraphics[width=0.23\textwidth]{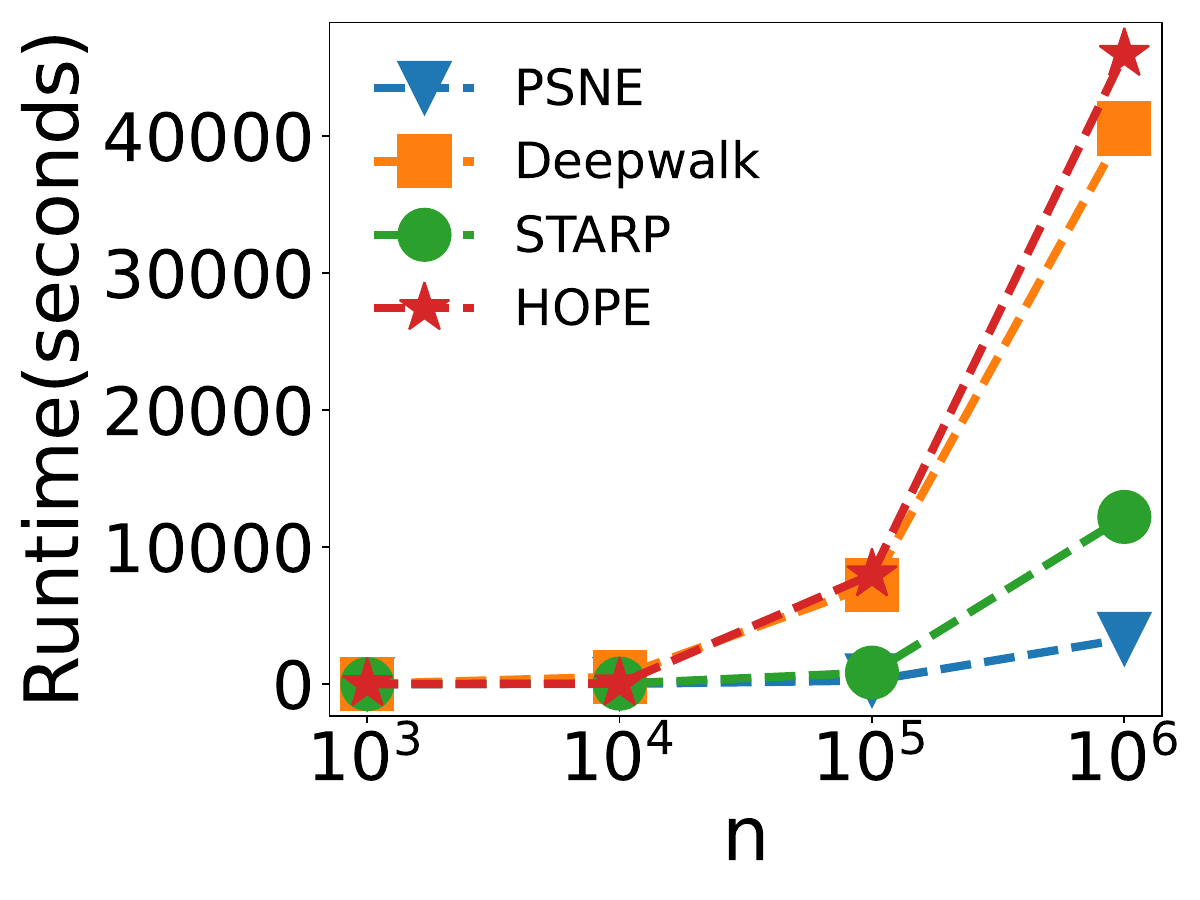}}
	\subfigure[\textit{BA}]{
		\includegraphics[width=0.23\textwidth]{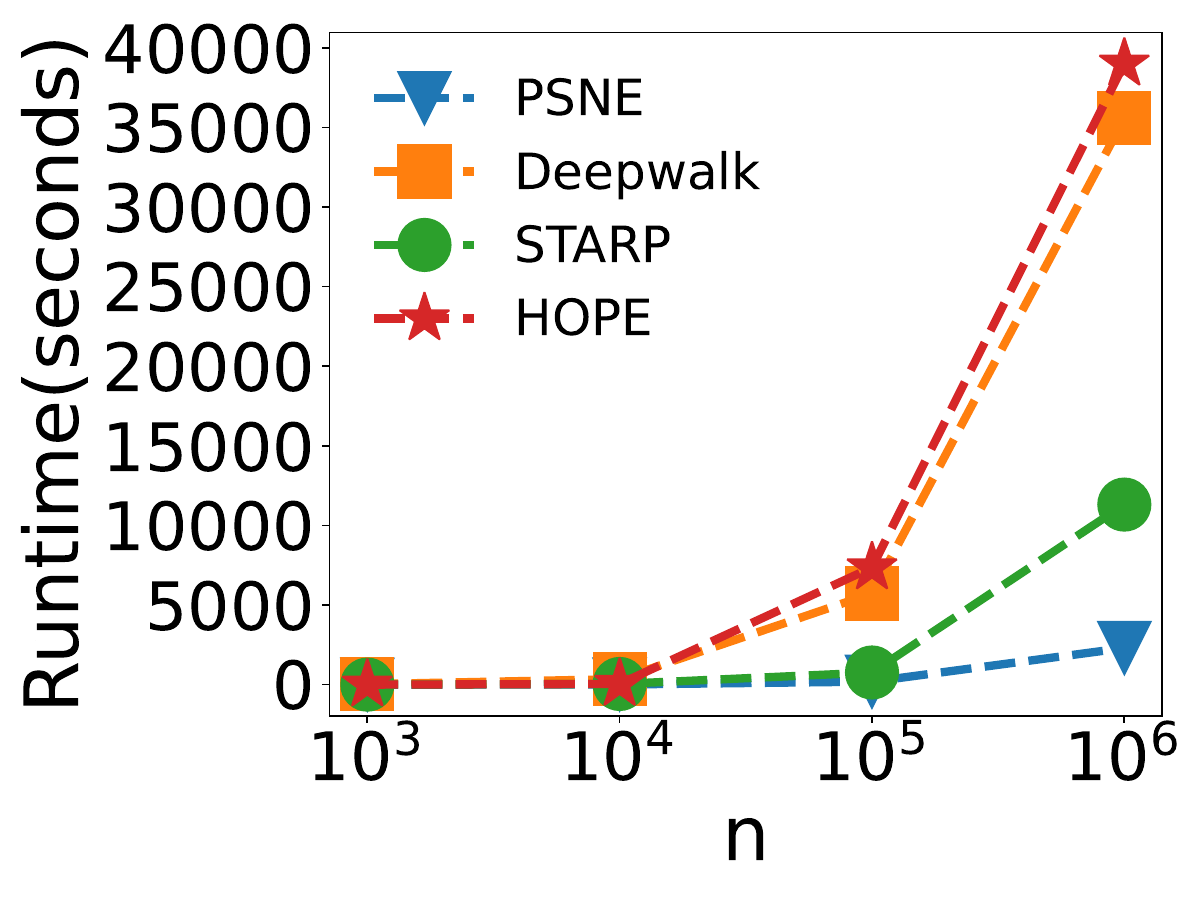}}
	\caption{Scalability testing on synthetic graphs.} 
\label{fig:time_syn}
\end{figure}

\begin{table}[t!]
 \caption{Ablation studies. MP (resp., NMP) is the corresponding method with (resp., without) multiple-perspective strategy. The best result is marked in \textbf{bold}.} 
	\centering
	\scalebox{1}{
\begin{tabular}{ccccccccccc}
				\toprule
 \multirow{2}{*}{\textbf{Model}} & \multicolumn{2}{c}{\textbf{PPI}} & \multicolumn{2}{c}{\textbf{Wikipedia}} & \multicolumn{2}{c}{\textbf{BlogCatalog}} \\
		\cmidrule(r){2-3}\cmidrule(r){4-5}\cmidrule(r){6-7}\cmidrule(r){8-9}\cmidrule(r){10-11}
	& NMP & MP	& NMP & MP	& NMP & MP	\\
\midrule
GraRep & 20.57 & 22.93& 50.51 & 53.60& 33.67& 37.21\\
HOPE & 20.72 & 23.73& \textbf{52.23 }& 53.46&34.25 &38.63\\
NetSMF & 23.1 & 23.64& 43.4 & 44.75& 39.33& 41.86\\
STRAP & 23.51 & 24.31& 51.87 & 52.78& 40.33& 41.41\\
ProNE & 23.84 & 24.11& 50.87 & 51.32& 40.73& 41.23\\
					\midrule
						PSNE & \textbf{24.07 }& \textbf{24.52}& 52.19 &\textbf{53.99 }&\textbf{41.02} & \textbf{43.14} \\
	\bottomrule	
\end{tabular}}
 \label{table:Ablation}
\end{table}

\begin{figure*}[t!]
    \centering
    \includegraphics[width=\textwidth]{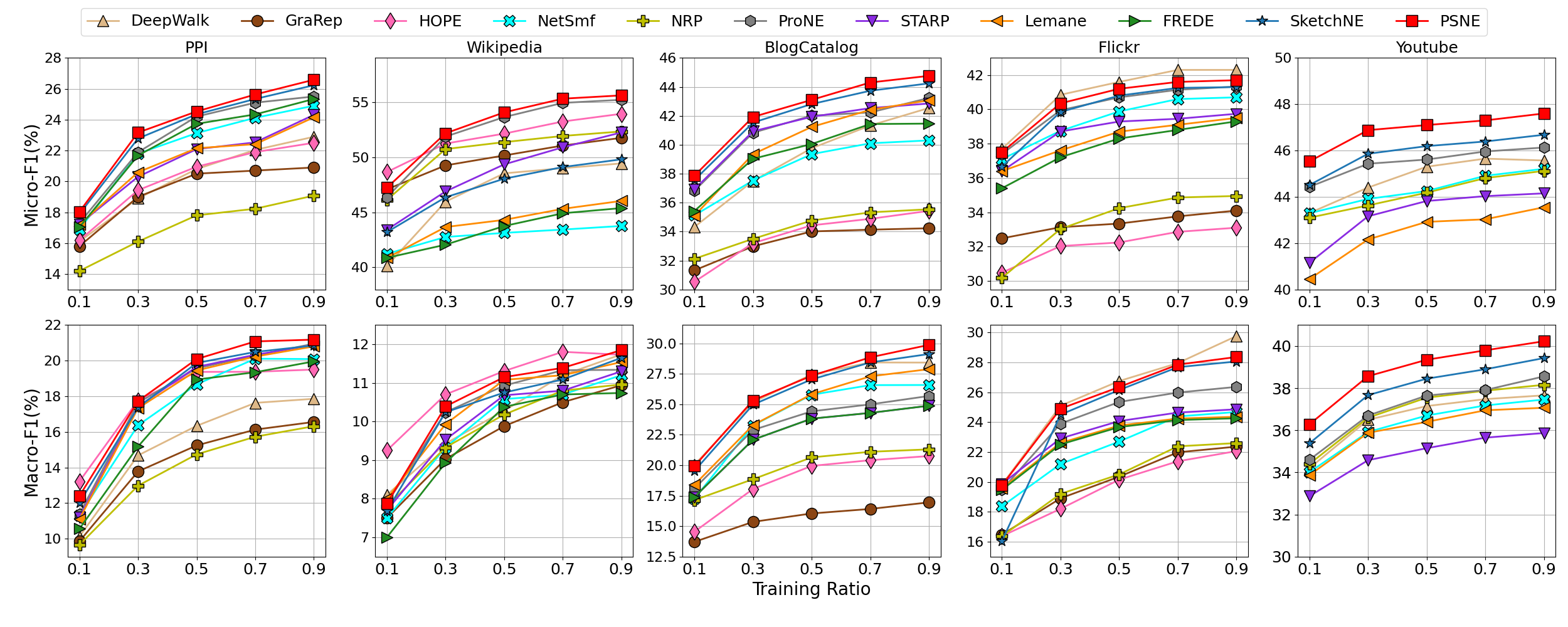}
    \caption{The performance of different network embedding methods}
    \label{fig:nodeclass}
\end{figure*}

\subsection{Ablation Studies}
To illustrate the impact of the multiple-perspective strategy on PSNE and other baselines, we report the results of the ablation study in Table \ref{table:Ablation}. As can be seen, on PPI, both GraRep and HOPE have a 2\%-3\% improvement, while NetSMF, STRAP, ProNE, and PSNE exhibit more modest gains of 0.3\%-0.8\%. On the Wikipedia and BlogCatalog datasets, all methods benefit significantly from the multiple-perspective strategy, with HOPE achieving the highest improvement on BlogCatalog (approximately 4.38\%) and other methods averaging around 1.2\% improvements. All experimental results were presented with a 50\% training rate.
\subsection{Sampling Quality Analysis}
We also observed that both PSNE and NetSMF use similar but completely different\footnote{New sampling probabilities and anonymous random walks are used in our solutions.} path sampling strategies to derive the node proximity matrix. Therefore, we will closely examine their differences. In particular, Figure \ref{fig:nodeclass} has revealed that NetSMF exhibits lower accuracy than PSNE. In addition, NetSMF requires significantly more path sampling to achieve acceptable accuracy. To illustrate this point, we compare the impact of sampling size on the F1 scores of NetSMF and PSNE. Following the previous methods \cite{DBLP:spectraljournals/corr/ChengCLPT15}, we also set the number of samples to $cTm$ (
$T$ is the path length) and vary $c$ to adjust the number of samples. As shown in Figure \ref{pvn},  under the same sampling scale, PSNE outperforms NetSMF, with a maximum difference of approximately 8\%. On top of that, when $c=30$, PSNE essentially meets the sampling quantity requirement, whereas NetSMF requires nearly 10 times (i.e., $c=300$) the sampling quantity of PSNE to achieve comparable performance. We believe there are two main reasons for this phenomenon: (1) NetSMF treats distant nodes and nearby nodes equally, missing the non-uniform higher-order structure information (Table 1). To this end, we use the $\alpha$-decay random 
walk (i.e., the PPR matrix in Equation (1)) to measure node proximity such that nearby nodes receive more attention. (2) 
NetSMF obtains proximity information of node pair ($v_i$, $v_j$) only by sampling path \{$v_i$ ...$v_j$\}, which requires more than $n^2$ paths to calculate the proximity of all node pairs. On the contrary, by the proposed multi-perspective strategy, $v_i's$ one-hop neighbors propagate their PPR values (w.r.t $v_j$) to $v_i$ to restore the PPR values (Equation \ref{fomula6}), allowing us to greatly reduce the number of samples while obtaining high-quality when compared to NetSMF.

\begin{figure}[t!]
	\centering
    \subfigure{\includegraphics[width=\linewidth]{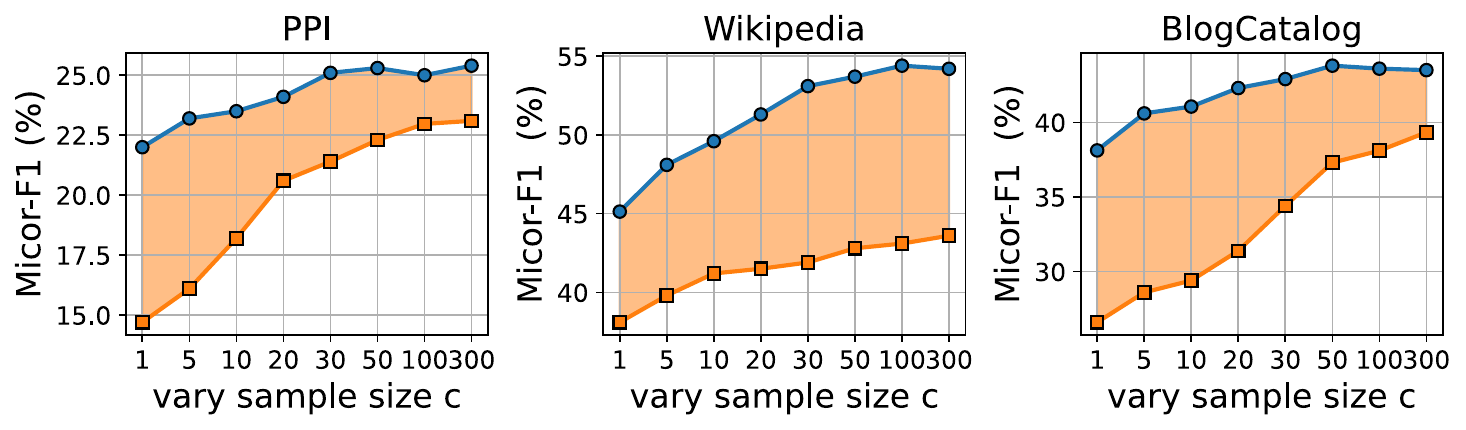}}\\	
\subfigure{\includegraphics[width=\linewidth]{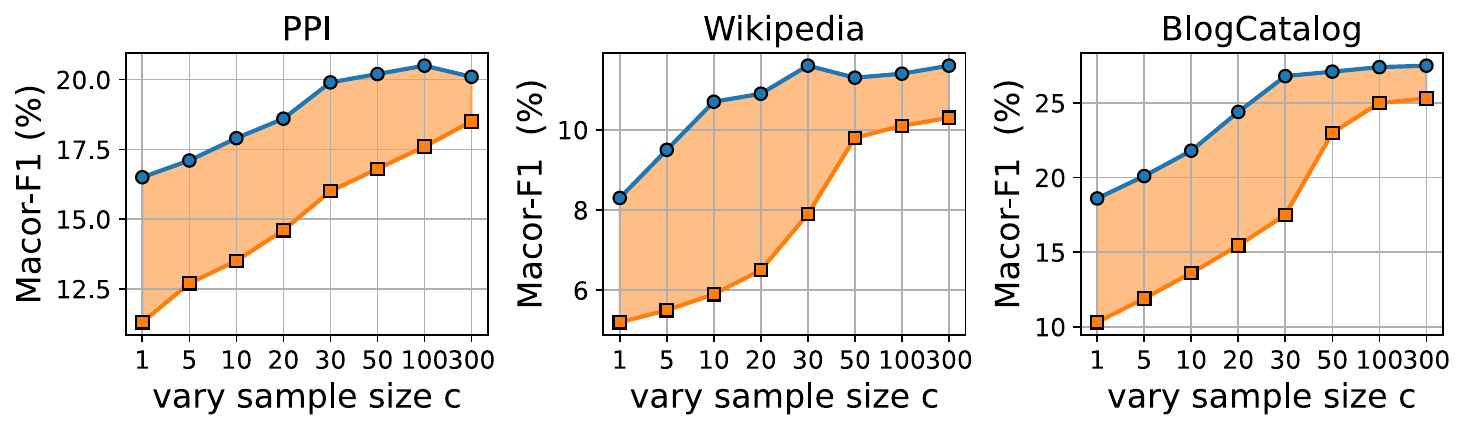}}\\
	\caption{Sampling quantity analysis (orange line and blue line represent NetSMF and PSNE, respectively).} \vspace{-0.5cm}
 \label{pvn}
\end{figure}

 \begin{figure*}[t!]
\centering
\subfigure[vary decay factor $\alpha$]{
\includegraphics[width=0.23\linewidth]{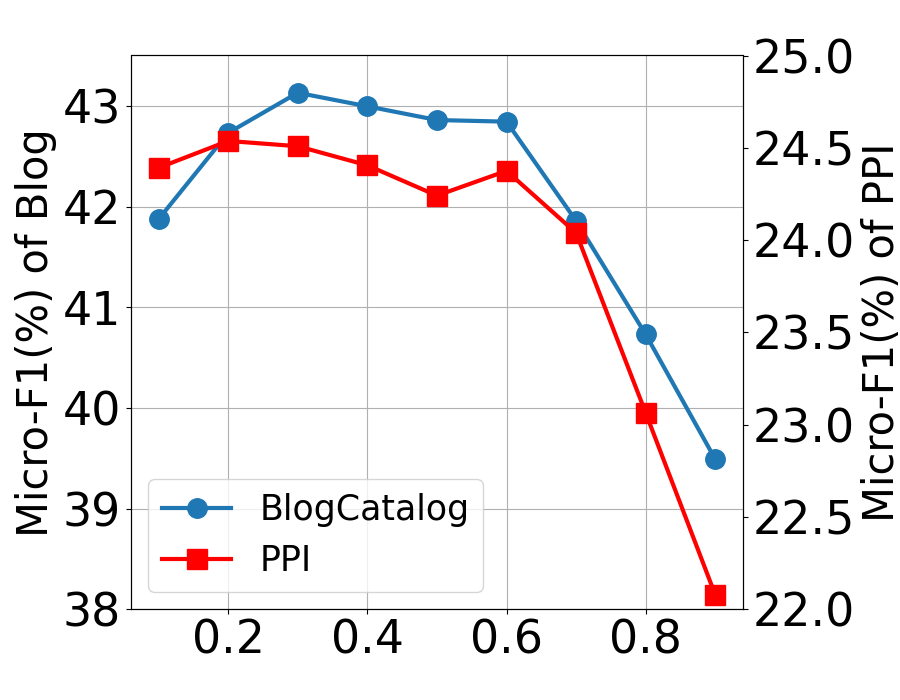}}
\subfigure[vary sample size $c$]{
\includegraphics[width=0.23\linewidth]{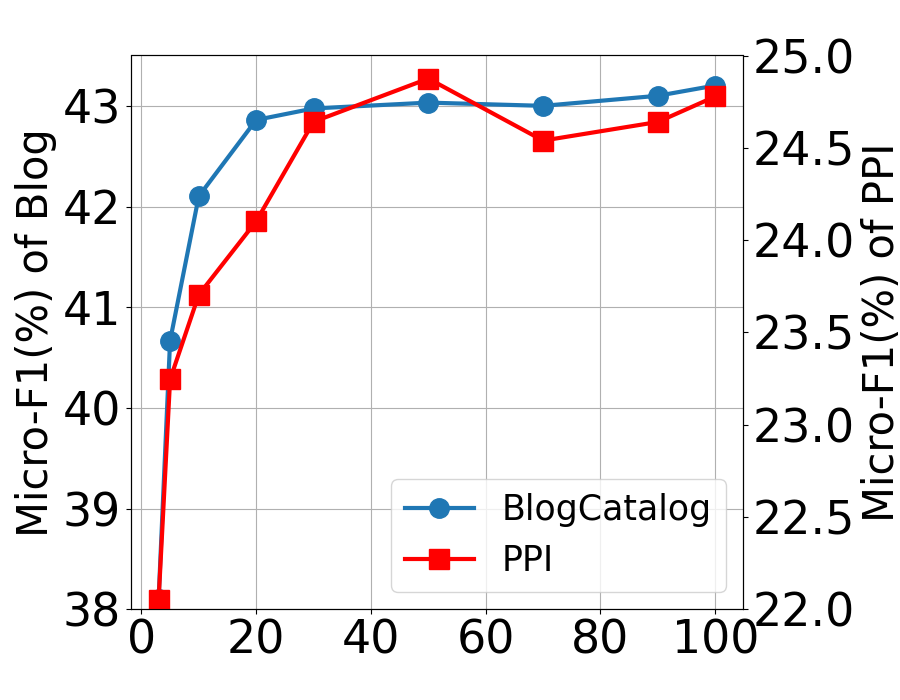}}
\subfigure[vary truncation order $T$]{
\includegraphics[width=0.23\linewidth]{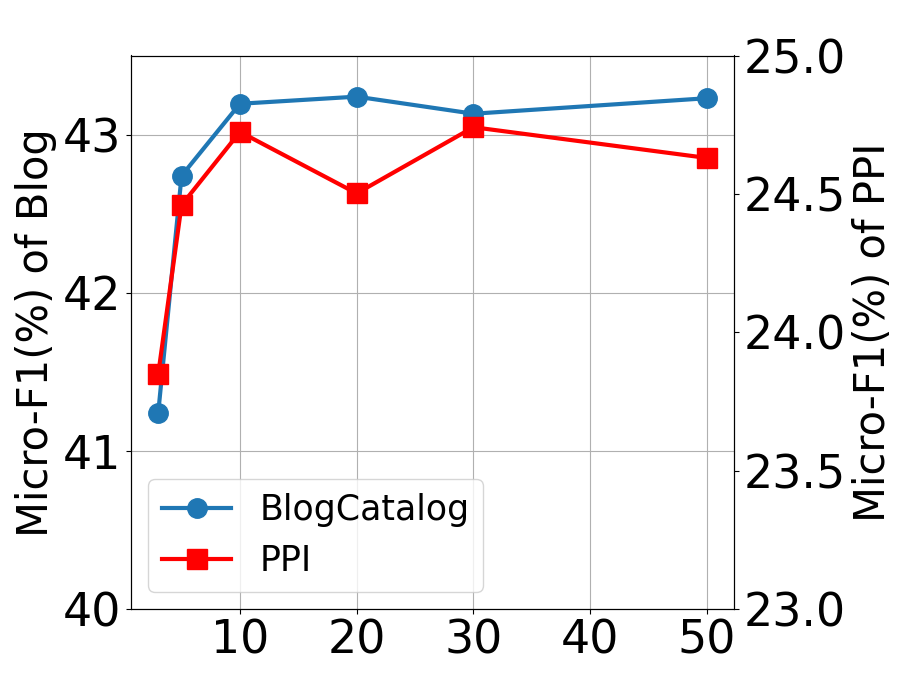}}
\subfigure[vary filter parameter $\mu$]{
\includegraphics[width=0.23\linewidth]{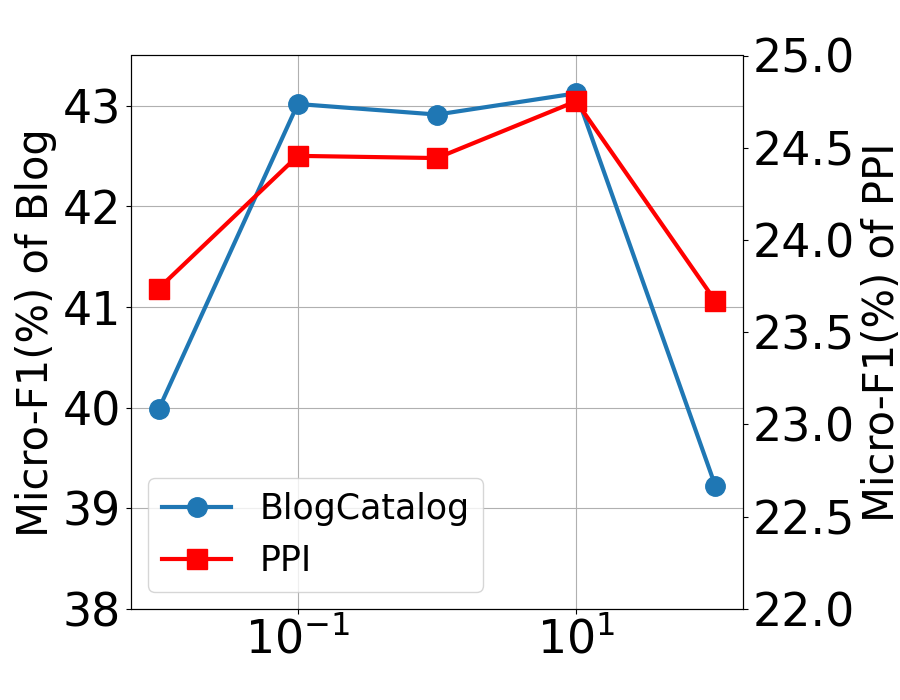}}
\caption{The performance of our proposed PSNE with varying parameters}
\label{fig:subfig}
\end{figure*}
\subsection{Parameter Analysis}
Following the previous methods \cite{DBLP:spectraljournals/corr/ChengCLPT15}, we also set $N=cTm$. Figure \ref{fig:subfig} only shows the effect of different parameter settings on Micro-F1 of PSNE in BlogCatalog and PPI, with comparable trends across Macro-F1 and other datasets. We have the following results: (1) By Figure 
\ref{fig:subfig}(a), Mirco-F1 increases first and then decreases with increasing  $\alpha$. This is because the parameter $\alpha$ controls how much of the graph is explored by our PSNE. Thus, when $\alpha$ is close to 1, the proximity matrix focuses on one-hop neighbors and thus preserves the adjacency information. As $\alpha$ approaches 0, the proximity matrix incorporates more information coming from multi-hop higher-order neighbors, resulting in good performance. (2) By Figure \ref {fig:subfig}(b) and \ref{fig:subfig}(c), we observe that Mirco-F1 increases first and then decreases with increasing  $c$ or $T$, and the optimal result is taken when $c=50$ or $T=10$. The reasons can be explained as follows: Although the larger $c$ or $T$, the closer the obtained $\boldsymbol{\tilde{\Pi}}$ by Algorithm \ref{alg:sparsifier} is to the exact PPR matrix $\boldsymbol{\Pi}$, the gain they bring leads to overfitting due to the small-world phenomenon \cite{Watts1998CollectiveDO} (e.g., numerous redundant and valueless neighbors lead to noise \cite{DBLP:journals/pacmmod/Li023}).
 (3) By Figure \ref {fig:subfig}(d), we know that both excessively large and small values of $\mu$ result in performance degradation. This observation suggests that small entries in the proximity matrix may introduce noise and impede the qualities of embedding vectors (Line 3 of Algorithm \ref{alg:final}).

\section{Conclusion}
This paper presents an efficient spectral sparsification algorithm PSNE that first devises a matrix polynomial sparser to \emph{directly} approximate the whole PPR matrix with theoretical guarantee, which avoids repeatedly computing each row or column of the PPR matrix. Then, PSNE introduces a simple yet effective multiple-perspective strategy to enhance further the representation power of the obtained sparse and coarse-grained PPR matrix. Finally, extensive empirical results show that PSNE can quickly obtain high-quality embedding vectors compared with ten competitors.

\section{ACKNOWLEDGMENTS}
The work was supported by (1) Fundamental Research Funds for the Central Universities under Grant SWU-KQ22028, (2) University Innovation Research Group of Chongqing (No. CXQT21005)
(3) the Fundamental Research Funds for the Central Universities (No. SWU-XDJH202303)
(4) the Natural Science Foundation of China (No. 72374173)
(5) the High Performance Computing clusters at Southwest University.

\section{Appendix}

\subsection{Approximate Error Analysis}
\begin{lemma}
\label{lemma_a1}
Let $\mathcal{L}=D^{-1/2}L_\beta(G)D^{-1/2}$ and the corresponding sparsifier $\mathcal{\widetilde{L}}=D^{-1/2}\widetilde{L}D^{-1/2}$, in which $\widetilde{L}$ is the Laplacian matrix of $\widetilde{G}$ (Line 8 of Algorithm 1). Then, all the singular values of $\mathcal{\widetilde{L}} - \mathcal{L}$ are less than $4\epsilon$.
\end{lemma}

\begin{proof}
According to Theorem 1, we have

\begin{center}
$\frac{1}{(1+\!\epsilon)} \!\cdot\! \boldsymbol{x}^{\top}\! L_{\boldsymbol{\beta}}(G) \boldsymbol{x} \!\leq \!\boldsymbol{x}^{\top}\! \tilde{\mathbf{L}} \boldsymbol{x}\!\leq\frac{1}{(1-\!\epsilon)}  \!\cdot \!\boldsymbol{x}^{\top}\!  L_{\boldsymbol{\beta}}(G) \boldsymbol{x}$
\end{center}

Let $x$=$D^{-1/2}y$,  we have
\begin{center}
$-\frac{\epsilon}{(1+\!\epsilon)} \!\cdot\! \boldsymbol{y}^{\top}\!\mathcal{L} \boldsymbol{y} \!\leq \!\boldsymbol{y}^{\top}\! (\mathcal{\widetilde{L}}\!-\!\mathcal{L}) \boldsymbol{y}\!\leq\frac{\epsilon}{(1-\!\epsilon)}  \!\cdot \!\boldsymbol{y}^{\top}\!  \mathcal{L} \boldsymbol{y}$
\end{center}

Since $\epsilon\leq 0.5$, we have
\begin{center}
$\lvert\boldsymbol{y}^{\top}\! (\mathcal{\widetilde{L}}\!-\!\mathcal{L}) \boldsymbol{y}\rvert \leq \frac{\epsilon}{(1-\epsilon)}  \cdot \!\boldsymbol{y}^{\top}  \mathcal{L} \boldsymbol{y} \leq 2\epsilon \boldsymbol{y}^{\top}  \mathcal{L} \boldsymbol{y}$ 
\end{center}

By Courant-Fisher Theorem \cite{DBLP:conf/focs/Spielman07}, we have 
 \begin{center}
 $\lvert\lambda_i(\mathcal{\widetilde{L}}\!-\!\mathcal{L}) \rvert \leq 2\epsilon \lambda_i(\mathcal{L}) < 4\epsilon $
 \end{center}
 where $i\in[n]$, $\lambda_i(A)$ is i-th largest eigenvalue of matrix $A$. This is because $\mathcal{L}$ is a normalized graph Laplacian matrix with eigenvalues in the interval [0, 2). 
\end{proof}
\begin{lemma}
\label{lemma_a2}
	\cite{2005Topics} If $B,C$ be two $n \times n$ symmetric matrices, for the decreasingly-ordered singular values $\lambda$ of $B,C$ and $BC$,
 \begin{center}
    $\lambda_{i+j-1}(BC)\leq \lambda_{i}(B)\times\lambda_{j}(BC)$
\end{center}
where $i \leq i,j \leq n$ and $i+j\leq n+1$.
\end{lemma}

Based on these lemmas, we give the detailed proofs of Theorem 4 and Theorem 5 as follows.

\stitle{The Proof of Theorem 4.} According to Equation 4, we have
\begin{align}	
 &{\Vert\Pi-{\Pi}^{\prime}\Vert}_F=\Vert{\alpha {\sum}_{i=T+1}^\infty(1-\alpha)^i{(D^{-1}A)}^i}\Vert_F\\
 &=\alpha {\sum}_{i=T+1}^\infty(1-\alpha)^i\sqrt{\sum_{j \in [n]}\lambda_{j}^{2}(({D^{-1}A})^i)}\\
 &\leq \alpha {\sum}_{i=T+1}^\infty(1-\alpha)^i\sqrt{\sum_{j \in [n]}\lambda_{j}^{2}({D^{-1}A})}\\
  &\leq \sqrt{n}(1-\alpha)^{T+1}
\end{align}
This is because $\lambda_{j}(({D^{-1}A})^i)\leq \lambda_{j}({D^{-1}A})*\lambda_{1}({D^{-1}A})*...*\lambda_{1}({D^{-1}A})\leq \lambda_{j}({D^{-1}A})$ by Lemma \ref{lemma_a2} and $\lambda_{s}({D^{-1}A})\in [0,1]$ for any $s \in [n]$. Furthermore, by Lemma \ref{lemma_a1} and Lemma \ref{lemma_a2}, we can know that 
${{\Pi}^{\prime}-\tilde{\Pi}}={\alpha}_{sum}D^{-1}(\tilde{\mathbf{L}}-{L}_{\boldsymbol{\alpha}}(G))={\alpha}_{sum}D^{-1/2}D^{-1/2}(\tilde{\mathbf{L}}-{L}_{\boldsymbol{\alpha}}(G))D^{-1/2}D^{1/2}$. Thus, $\lambda_i({{\Pi}^{\prime}-\tilde{\Pi}}) \leq {\alpha}_{sum}\lambda_1(D^{-1/2})\lambda_i(\mathcal{\widetilde{L}} -\mathcal{L}) \lambda_1(D^{1/2})
    \leq 4\epsilon\cdot {\alpha}_{sum}$.
As a result, $\Vert{{\Pi}^{\prime}-\tilde{\Pi}\Vert}_F=\sqrt{\sum_{i\in[n]}\lambda_i^2({\Pi}^{\prime}-\tilde{\Pi})} \leq 4\epsilon\cdot {\alpha}_{sum} \sqrt{n}$.  Based on these analyses, we have 
\begin{align}
&\footnotesize{\Vert\Pi-\tilde{\Pi}\Vert}_F \leq {\Vert\Pi-\Pi^{'}\Vert}_F+{\Vert\Pi^{'}-\tilde{\Pi}\Vert}_F\\
& \footnotesize\leq 
\sqrt{n}(1-\alpha)^{T+1}+4\epsilon\cdot {\alpha}_{sum} \sqrt{n}
\end{align} 

Therefore, we have completed the proof of Theorem 4. 

\stitle{The Proof of Theorem 5.} 
To simplify the proof, let's assume that  $\lambda_{hi}=1/(\sqrt{(d_{h}+1)}\sqrt{(d_{i}+1)})$ and $\lambda_{ii}=1/(d_{i}+1)$. Therefore, by Definition 10, we can know that
$M_{\boldsymbol{\Pi}}=\boldsymbol{{\tilde{D}^{-1/2}\tilde{A}\tilde{D}^{-1/2}}\Pi}$ and $M_{\boldsymbol{\tilde{\Pi}}}=\boldsymbol{{\tilde{D}^{-1/2}\tilde{A}\tilde{D}^{-1/2}}\tilde{\Pi}}$, in which $\tilde{A}=A+I$ (I is the identify matrix) and $\tilde{D}$ is degree matrix of $\tilde{A}$. Since $max(0,\log(xn\mu))$ is l-Lipchitz w.r.t Frobenius norm, we have 
\begin{align}
&{\Vert\boldsymbol{\sigma}_{\mu}(M_{\boldsymbol{\Pi}})-\boldsymbol{\sigma}_{\mu}(M_{\boldsymbol{\tilde{\Pi}}})\Vert}_F \\
&={\Vert\boldsymbol{\sigma}_{\mu}(\boldsymbol{{\tilde{D}^{-1/2}\tilde{A}\tilde{D}^{-1/2}}\Pi})-\boldsymbol{\sigma}_{\mu}(\boldsymbol{{\tilde{D}^{-1/2}\tilde{A}\tilde{D}^{-1/2}}\tilde{\Pi}})\Vert}_F \\
&\leq {\Vert\boldsymbol{{\tilde{D}^{-1/2}\tilde{A}\tilde{D}^{-1/2}}(\Pi-\tilde{\Pi})}\Vert}_F  \\
&\leq 
n((1-\alpha)^{T+1}+4\epsilon\cdot {\alpha}_{sum})
\end{align}

Therefore, we have completed the proof of Theorem 5. 

Since $(1-\alpha)^{T+1}+4\epsilon\cdot {\alpha}_{sum}$ is a constant due to $0 <\epsilon\leq 0.5$, ${\alpha}_{sum}\leq 1-\alpha$, and $(1-\alpha)^{T+1}\leq 1$, we have ${\Vert\Pi-\tilde{\Pi}\Vert}_F= 
O(\sqrt{n})$, $ {\Vert\boldsymbol{\sigma}(M_{\boldsymbol{\Pi}}, \mu)-\boldsymbol{\sigma}(M_{\boldsymbol{\tilde{\Pi}}}, \mu)\Vert}_F =O(n)$.

\bibliographystyle{ACM-Reference-Format}
\balance
\bibliography{main}

\end{document}